\documentclass[11pt]{article}
\usepackage[T1]{fontenc}
\usepackage[latin9]{inputenc}
\usepackage{amsmath}
\usepackage{amsthm}
\usepackage{amssymb}
\usepackage{graphicx}
\usepackage[unicode=true,
 bookmarks=false,
 breaklinks=false,pdfborder={0 0 1},backref=false,colorlinks=false]
 {hyperref}

\makeatletter

\providecommand{\tabularnewline}{\\}

\theoremstyle{plain}
\newtheorem{thm}{\protect\theoremname}
  \theoremstyle{definition}
  \newtheorem{defn}[thm]{\protect\definitionname}
  \theoremstyle{plain}
  \newtheorem{cor}[thm]{\protect\corollaryname}
  \theoremstyle{plain}
  \newtheorem{prop}[thm]{\protect\propositionname}
  \theoremstyle{remark}
  \newtheorem{rem}[thm]{\protect\remarkname}
  \theoremstyle{plain}
  \newtheorem{lem}[thm]{\protect\lemmaname}

\usepackage[margin=1in]{geometry}
\usepackage{times}
\newtheorem{assume}{Assumption}
\usepackage{matlab-prettifier}
\usepackage{thmtools}
\usepackage{thm-restate}

\usepackage{listings}
\@ifundefined{showcaptionsetup}{}{%
 \PassOptionsToPackage{caption=false}{subfig}}
\usepackage{subfig}
\makeatother

\usepackage{listings}
  \providecommand{\corollaryname}{Corollary}
  \providecommand{\definitionname}{Definition}
  \providecommand{\lemmaname}{Lemma}
  \providecommand{\propositionname}{Proposition}
  \providecommand{\remarkname}{Remark}
\providecommand{\theoremname}{Theorem}

\begin{document}

\title{Large-Scale Sparse Inverse Covariance Estimation via Thresholding
and Max-Det Matrix Completion\thanks{Proceedings of the 35-th International Conference on Machine Learning,
Stockholm, Sweden, PMLR 80, 2018. Copyright 2018 by the author(s).
This work was supported by the ONR grants N00014-17-1-2933 and N00014-15-1-2835,
DARPA grant D16AP00002, and AFOSR grant FA9550-17-1-0163.}}

\author{Richard Y. Zhang\thanks{Department of Industrial Engineering and Operations Research, University
of California, Berkeley, CA USA. (\url{ryz@berkeley.edu}, \url{fattahi@berkeley.edu})} \and Salar Fattahi\footnotemark[2] \and Somayeh Sojoudi\thanks{Department of Electrical Engineering and Computer Sciences / Department
of Mechanical Engineering, University of California, Berkeley, CA
USA. (\url{sojoudi@berkeley.edu})}}
\maketitle
\begin{abstract}
The sparse inverse covariance estimation problem is commonly solved
using an $\ell_{1}$-regularized Gaussian maximum likelihood estimator
known as ``graphical lasso'', but its computational cost becomes
prohibitive for large data sets. A recent line of results showed\textendash under
mild assumptions\textendash that the graphical lasso estimator can
be retrieved by soft-thresholding the sample covariance matrix and
solving a maximum determinant matrix completion (MDMC) problem. This
paper proves an extension of this result, and describes a Newton-CG
algorithm to efficiently solve the MDMC problem. Assuming that the
thresholded sample covariance matrix is sparse with a sparse Cholesky
factorization, we prove that the algorithm converges to an $\epsilon$-accurate
solution in $O(n\log(1/\epsilon))$ time and $O(n)$ memory. The algorithm
is highly efficient in practice: we solve the associated MDMC problems
with as many as 200,000 variables to 7-9 digits of accuracy in less
than an hour on a standard laptop computer running MATLAB. 
\end{abstract}

\section{Introduction}

\global\long\def\S{\mathbb{S}}
 \global\long\def\R{\mathbb{R}}
 \global\long\def\tr{\mathrm{tr}\,}
 \global\long\def\mach{\mathrm{mach}}
 \global\long\def\cond{\mathrm{cond}\,}
 \global\long\def\dom{\mathrm{dom}\,}
 \global\long\def\K{\mathcal{K}}
 \global\long\def\A{\mathbf{A}}
 \global\long\def\vector{\mathrm{vec}\,}
 Consider the problem of estimating an $n\times n$ covariance matrix
$\Sigma$ (or its inverse $\Sigma^{-1}$) of a $n$-variate probability
distribution from $N$ independently and identically distributed samples
$\mathbf{x}_{1},\mathbf{x}_{2},\ldots,\mathbf{x}_{N}$ drawn from
the same probability distribution. In applications spanning from computer
vision, natural language processing, to economics~\cite{Li94,Manning99,Steve93},
the matrix $\Sigma^{-1}$ is often \emph{sparse}, meaning that its
matrix elements are mostly zero. For Gaussian distributions, the statistical
interpretation of sparsity in $\Sigma^{-1}$ is that most of the variables
are pairwise conditionally independent~\cite{Mein06,Ming07,Friedman08,Banerjee08}.

Imposing sparsity upon $\Sigma^{-1}$ can regularize the associated
estimation problem and greatly reduce the number of samples required.
This is particularly important in \textit{\emph{high-dimensional}}
settings where $n$ is large, often significantly larger than the
number of samples $N\ll n$. One popular approach regularizes the
associated maximum likelihood estimation (MLE) problem by a sparsity-promoting
$\ell_{1}$ term, as in 
\begin{equation}
\underset{X\succ0}{\text{minimize }}\tr CX-\log\det X+\lambda\sum_{i=1}^{n}\sum_{j=1}^{n}|X_{i,j}|.\label{eq:gl}
\end{equation}
Here, $C=\frac{1}{N}\sum_{i=1}^{N}(\mathbf{x}_{i}-\bar{\mathbf{x}})(\mathbf{x}_{i}-\bar{\mathbf{x}})^{T}$
is the sample covariance matrix with sample mean $\bar{\mathbf{x}}=\frac{1}{N}\sum_{i=1}^{N}\mathbf{x}_{i}$,
and $X$ is the resulting estimator for $\Sigma^{-1}$. This approach,
commonly known as the \emph{graphical lasso}~\cite{Friedman08},
is known to enjoy a number of statistical guarantees~\cite{Rothman08,Martin11},
some of which are direct extensions of earlier work on the classical
lasso~\cite{Martin08,Martin082,Martin09,Huang10}. A variation on
this theme is to only impose the $\ell_{1}$ penalty on the off-diagonal
elements of $X$, or to place different weights $\lambda$ on the
elements of the matrix $X$, as in the classical weighted lasso.

While the $\ell_{1}$-regularized problem (\ref{eq:gl}) is technically
convex, it is commonly considered intractable for large-scale datasets.
The decision variable is an $n\times n$ matrix, so simply fitting
all $O(n^{2})$ variables into memory is already a significant issue.
General-purpose algorithms have either prohibitively high complexity
or slow convergence. In practice, (\ref{eq:gl}) is solved using problem-specific
algorithms. The state-of-the-art include GLASSO~\cite{Friedman08},
QUIC~\cite{Hsieh14}, and its ``big-data'' extension BIG-QUIC~\cite{hsieh2013big}.
These algorithms use between $O(n)$ and $O(n^{3})$ time and between
$O(n^{2})$ and $O(n)$ memory per iteration, but the number of iterations
needed to converge to an accurate solution can be very large.

\subsection{Graphical lasso, soft-thresholding, and MDMC}

The high practical cost of graphical lasso has inspired a number of
heuristics, which enjoy less guarantees but are significantly cheaper
to use. Indeed, heuristics are often the only viable option once $n$
exceeds the order of a few tens of thousands. 

One simple idea is to \emph{threshold} the sample covariance matrix
$C$: to examine all of its elements and keep only the ones whose
magnitudes exceed some threshold. Thresholding can be fast\textemdash even
for very large-scale datasets\textemdash because it is embarassingly
parallel; its quadratic $O(n^{2})$ total work can be spread over
thousands or millions of parallel processors, in a GPU or distributed
on cloud servers. When the number of samples $N$ is small, i.e. $N\ll n$,
thresholding can also be performed using $O(n)$ memory, by working
directly with the $n\times N$ centered matrix-of-samples $\mathbf{X}=[\mathbf{x}_{1}-\bar{\mathbf{x}},\mathbf{x}_{2}-\bar{\mathbf{x}},\ldots,\mathbf{x}_{N}-\bar{\mathbf{x}}]$
satisfying $C=\frac{1}{N}\mathbf{X}\mathbf{X}^{T}$.

In a recent line of work~\cite{Mazumdar12,Sojoudi16,Salar17,Salar18},
the simple heuristic of thresholding was shown to enjoy some surprising
guarantees. In particular,~\cite{Sojoudi16,Salar17} proved that
when the lasso weight is imposed over only the off-diagonal elements
of $X$ that\textemdash under some assumptions\textemdash the \emph{sparsity
pattern} of the associated graphical lasso estimator can be recovered
by performing a soft-thresholding operation on $C$, as in 
\begin{equation}
(C_{\lambda})_{i,j}=\begin{cases}
C_{i,j} & i=j,\\
C_{i,j}-\lambda & C_{i,j}>\lambda,\;i\not=j,\\
0 & |C_{i,j}|\le\lambda\;i\not=j,\\
C_{i,j}+\lambda & -\lambda\le C_{i,j}\,i\not=j,
\end{cases}\label{thresh}
\end{equation}
and recovering
\begin{equation}
G=\{(i,j)\in\{1,\ldots,n\}^{2}:(C_{\lambda})_{i,j}\ne0\}.\label{eq:pattern}
\end{equation}
The associated graph (also denoted as $G$ when there is no ambiguity)
is obtained by viewing each nonzero element $(i,j)$ in $G$ as an
edge between the $i$-th and $j$-th vertex in an undirected graph
on $n$ nodes. Moreover, they showed that the estimator $X$ can be
recovered by solving a version of (\ref{eq:gl}) in which the sparsity
pattern $G$ is explicitly imposed, as in 
\begin{align}
\underset{X\succ0}{\text{minimize }} & \tr C_{\lambda}X-\log\det X\label{eq:mdmc1}\\
\text{subject to } & X_{i,j}=0\quad\forall(i,j)\notin G.\nonumber 
\end{align}
Recovering the exact value of $X$ (and not just its sparsity pattern)
is important because it provides a shrinkage MLE when the true MLE
is ill-defined; for Gaussian fields, its nonzero values encode the
partial correlations between variables. Problem (\ref{eq:mdmc1})
is named the \emph{maximum determinant matrix completion} (MDMC) in
the literature, for reasons explained below. The problem has a recursive
closed-form solution whenever the graph of $G$ is \emph{acyclic}
(i.e. a tree or forest)~\cite{Salar17}, or more generally, if it
is \emph{chordal}~\cite{Salar18}. It is worth emphasizing that the
closed-form solution is \emph{extremely} fast to evaluate: a chordal
example in~\cite{Salar18} with 13,659 variables took just $\approx5$
seconds to solve on a laptop computer.

The assumptions needed for graphical lasso to be equivalent to thresolding
are hard to check but relatively mild. Indeed, \cite{Salar17}~proves
that they are automatically satisfied whenever $\lambda$ is sufficiently
large relative to the sample covariance matrix. Their numerical study
found ``sufficiently large'' to be a fairly loose criterion in practice,
particularly in view of the fact that large values of $\lambda$ are
needed to induce a sufficiently sparse estimate of $\Sigma^{-1}$,
e.g. with $\approx10n$ nonzero elements.

However, the requirement for $G$ to be chordal is very strong. Aside
from trivial chordal graphs like trees and cliques, thresholding will
produce a chordal graph with probability zero. When $G$ is nonchordal,
no closed-form solution exists, and one must resort to an iterative
algorithm. The state-of-the-art for nonchordal MDMC is to embed the
nonchordal graph within a chordal graph, and to solve the resulting
problem as a semidefinite program using an interior-point method~\cite{dahl2008covariance,Andersen2010}.

\subsection{Main results}

The purpose of this paper is two-fold. First, we derive an extension
of the guarantees derived by~\cite{Mazumdar12,Sojoudi16,Salar17,Salar18}
for a slightly more general version of the problem that we call \emph{restricted}
graphical lasso (RGL): 
\begin{align}
\hat{X}=\underset{X\succ0}{\text{minimize }} & \tr CX-\log\det X+\sum_{i=1}^{n}\sum_{j=i+1}^{n}\lambda_{i,j}|X_{i,j}|\label{eq:rgl}\\
\text{subject to } & X_{i,j}=0\qquad\forall(i,j)\notin H.\nonumber 
\end{align}
In other words, RGL is (\ref{eq:gl}) penalized by a weighted lasso
penalty $\lambda_{i,j}$ on the off-diagonals, and with an \emph{a
priori} sparsity pattern $H$ imposed as an additional constraint.
We use the sparsity pattern $H$ to incorporate \textit{\emph{prior
information}} on the structure of the graphical model. For example,
if the sample covariance $C$ is collected over a graph, such as a
communication system or a social network, then far-away variables
can be assumed as pairwise conditionally independent~\cite{Park99,Honorio09,Croft10}.
Including these neighborhood relationships into $H$ can regularize
the statistical problem, as well as reduce the numerical cost for
a solution.

In Section~\ref{sec:rgl}, we describe a procedure to transform RGL
(\ref{eq:rgl}) into MDMC (\ref{eq:mdmc1}), in the same style as
prior results by~\cite{Salar17,Salar18} for graphical lasso. More
specifically, we soft-threshold the sample covariance $C$ and then
project this matrix onto the sparsity pattern $H$. We give conditions
for the resulting sparsity pattern to be equivalent to the one obtained
by solving (\ref{eq:rgl}). Furthermore, we prove that the resulting
estimator $X$ can be recovered by solving the same MDMC problem (\ref{eq:mdmc1})
with $C_{\lambda}$ appropriately modified.

The second purpose is to describe an efficient algorithm to solve
MCDC when the graph $G$ is \emph{nonchordal}, based on the chordal
embedding approach of~\cite{dahl2008covariance,Andersen2010,andersen2013logarithmic}.
We embed $G$ within a chordal $\tilde{G}\supset G$, to result in
a convex optimization problem over $\S_{\tilde{G}}^{n}$, the space
of real symmetric matrices with sparsity pattern $\tilde{G}$. This
way, the constraint $X\in\S_{\tilde{G}}^{n}$ is \emph{implicitly}
imposed, meaning that we simply ignore the nonzero elements not in
$\tilde{G}$. Next, we solve an optimization problem on $\S_{\tilde{G}}^{n}$
using a custom Newton-CG method\footnote{The MATLAB source code for our solver can be found at \url{http://alum.mit.edu/www/ryz}}.
The main idea is to use an inner conjugate gradients (CG) loop to
solve the Newton subproblem of an outer Newton's method. The actual
algorithm has a number of features designed to exploit problem structure,
including the sparse chordal property of $\tilde{G}$, duality, and
the ability for CG and Newton to converge superlinearly; these are
outlined in Section~\ref{sec:Proposed-Algorithm}.

Assuming that the chordal embedding is sparse with $|\tilde{G}|=O(n)$
nonzero elements, we prove in Section~\ref{sub:Proof-of-Linear},
that our algorithm converges to an $\epsilon$-accurate solution of
MDMC (\ref{eq:mdmc1}) in 
\begin{equation}
O(n\cdot\log\epsilon^{-1}\cdot\log\log\epsilon^{-1})\text{ time and }O(n)\text{ memory.}\label{eq:lintime}
\end{equation}
Most importantly, the algorithm is highly efficient in practice. In
Section~\ref{sec:Numerical-Results}, we present computation results
on a suite of test cases. Both synthetic and real-life graphs are
considered. Using our approach, we solve sparse inverse covariance
estimation problems containing as many as 200,000 variables, in less
than an hour on a laptop computer.

\subsection{Related Work}

\textbf{Graphical lasso with prior information.} A number of approaches
are available in the literature to introduce \textit{\emph{prior information}}
to graphical lasso. The weighted version of graphical lasso mentioned
before is an example, though RGL will generally be more efficient
to solve due to a reduction in the number of variables. \cite{Ortega17}~introduced
a class of graphical lasso in which the true graphical model is assumed
to have Laplacian structure. This structure commonly appears in signal
and image processing~\cite{Milanfar13}. For the \emph{a priori}
graph-based correlation structure described above, \cite{Maxim15}~introduced
a \emph{pathway} graphical lasso method similar to RGL.

\textbf{Algorithms for graphical lasso.} Algorithms for graphical
lasso are usually based on some mixture of Newton~\cite{oztoprak2012newton},
proximal Newton~\cite{hsieh2013big,Hsieh14}, iterative thresholding~\cite{rolfs2012iterative},
and (block) coordinate descent~\cite{Friedman08,treister2014block}.
All of these suffer fundamentally from the need to keep track and
act on all $O(n^{2})$ elements in the matrix $X$ decision variable.
Even if the final solution matrix were sparse with $O(n)$ nonzeros,
it is still possible for the algorithm to traverse through a ``dense
region'' in which the iterate $X$ must be fully dense. Thresholding
heuristics have been proposed to address issue, but these may adversely
affect the outer algorithm and prevent convergence. It is generally
impossible to guarantee a figure lower than $O(n^{2})$ time per-iteration,
even if the solution contains only $O(n)$ nonzeros. Most of the algorithms
mentioned above actually have worst-case per-iteration costs of $O(n^{3})$.

\textbf{Graphical lasso via thresholding.} The elementary estimator
for graphical models (EE-GM)~\cite{Yang14} is another thresholding-based
low-complexity method that is able to recover the actual graphical
lasso estimator. Both EE-GM and our algorithm have a similar level
of performance in practice, because both algorithm are bottlenecked
by the initial thresholding step, which is a quadratic $O(n^{2})$
time operation.

\textbf{Algorithms for MDMC.} Our algorithm is inspired by a line
of results~\cite{dahl2008covariance,Andersen2010,andersen2013logarithmic,li2017inexact}
for minimizing the log-det penalty on chordal sparsity patterns, culminating
in the CVXOPT package~\cite{andersen2013cvxopt}. These are Newton
algorithms that solve the Newton subproblem by explicitly forming
and factoring the fully-dense Newton matrix. When $|\tilde{G}|=O(n)$,
these algorithms cost $O(nm^{2}+m^{3})$ time and $O(m^{2})$ memory
per iteration, where $m$ is the number of edges added to $G$ to
yield the chordal $\tilde{G}$. In practice, $m$ is usually a factor
of 0.1 to 20 times $n$, so these algorithms are cubic $O(n^{3})$
time and $O(n^{2})$ memory. Our algorithm solves the Newton subproblem
iteratively using CG. We prove that CG requires just $O(n)$ time
to compute the Newton direction to machine precision (see Section~\ref{sub:Proof-of-Linear}).
In practice, CG converges much faster than its worst-case bound, because
it is able to exploit eigenvalue clustering to achieve superlinear
convergence. 

\subsection*{Notations}

Let $\mathbb{R}^{n}$ be the set of $n\times1$ real vectors, and
$\S^{n}$ be the set of $n\times n$ real symmetric matrices. (We
denote $x\in\R^{n}$ using lower-case, $X\in\S^{n}$ using upper-case,
and index the $(i,j)$-th element of $X$ as $X_{i,j}$.) We endow
$\S^{n}$ with the usual matrix inner product $X\bullet Y=\tr XY$
and Euclidean (i.e. Frobenius) norm $\|X\|_{F}^{2}=X\bullet X$. Let
$\S_{+}^{n}\subset\S^{n}$ and $\S_{++}^{n}\subset\S_{+}^{n}$ be
the associated set of positive semidefinite and positive definite
matrices. We will frequently write $X\succeq0$ to mean $X\in\S_{+}^{n}$
and write $X\succ0$ to mean $X\in\S_{++}^{n}$. Given a sparsity
pattern $G$, we define $\S_{G}^{n}\subseteq\S^{n}$ as the set of
$n\times n$ real symmetric matrices with this sparsity pattern.

\section{\label{sec:rgl}Restricted graphical lasso, soft-thresholding, and
MDMC}

Let $P_{H}(X)$ denote the projection operator from $\S^{n}$ onto
$\S_{H}^{n}$, i.e. by setting all $X_{i,j}=0$ if $(i,j)\notin H$.
Let $C_{\lambda}$ be the sample covariance matrix $C$ individually
soft-thresholded by $[\lambda_{i,j}]$, as in 
\begin{equation}
(C_{\lambda})_{i,j}=\begin{cases}
C_{i,j} & i=j,\\
C_{i,j}-\lambda_{i,j} & C_{i,j}>\lambda_{i,j},\;i\not=j,\\
0 & |C_{i,j}|\le\lambda_{i,j}\;i\not=j,\\
C_{i,j}+\lambda_{i,j} & -\lambda_{i,j}\le C_{i,j}\,i\not=j,
\end{cases}\label{thresh2}
\end{equation}
In this section, we state the conditions for $P_{H}(C_{\lambda})$\textemdash the
projection of the soft-thresholded matrix $C_{\lambda}$ in~(\ref{thresh2})
onto $H$\textemdash to have the same sparsity pattern as the RGL
estimator $\hat{X}$ in (\ref{eq:rgl}). Furthermore, the estimator
$\hat{X}$ can be explicitly recovered by solving the MDMC problem
(\ref{eq:mdmc1}) while replacing $C_{\lambda}\gets P_{H}(C_{\lambda})$
and $G\gets P_{H}(G)$. For brevity, all proofs and remarks are omitted;
these can be found in the appendix.

Before we state the exact conditions, we begin by adopting the some
definitions and notations from the literature. 
\begin{defn}
\cite{Salar17}\label{def:dd1} Given a matrix $M\in\S^{n}$, define
$G_{M}=\{(i,j):M_{i,j}\ne0\}$ as its sparsity pattern. Then $M$
is called \textbf{inverse-consistent} if there exists a matrix $N\in\S^{n}$
such that \begin{subequations} 
\begin{align}
 & M+N\succ0\\
 & N=0\qquad\forall(i,j)\in G_{M}\\
 & (M+N)^{-1}\in\mathbb{S}_{G_{M}}^{n}
\end{align}
\end{subequations} The matrix $N$ is called an \textbf{inverse-consistent
complement } of $M$ and is denoted by $M^{(c)}$. Furthermore, $M$
is called \textbf{sign-consistent} if for every $(i,j)\in G_{M}$,
the $(i,j)$-th elements of $M$ and $(M+M^{(c)})^{-1}$ have opposite
signs. 
\end{defn}
Moreover, we take the usual matrix max-norm to exclude the diagonal,
as in $\|M\|_{\max}=\max_{i\not=j}|M_{ij}|,$ and adopt the $\beta(G,\alpha)$
function defined with respect to the sparsity pattern $G$ and scalar
$\alpha>0$ 
\begin{align*}
\beta(G,\alpha)=\max_{M\succ0} & \|M^{(c)}\|_{\max}\\
\text{s.t. } & M\in\S_{G}^{n}\text{ and }\|M\|_{\max}\le\alpha\\
 & M_{i,i}=1\quad\forall i\in\{1,\ldots,n\}\\
 & M\text{ is inverse-consistent.}
\end{align*}
We are now ready to state the conditions for soft-thresholding to
be equivalent to RGL. 

\begin{restatable}{thm}{rglthm}\label{thm:tt1}Define $C_{\lambda}$
as in (\ref{thresh2}), define $C_{H}=P_{H}(C_{\lambda})$ and let
$G_{H}=\{(i,j):(C_{H})_{i,j}\ne0\}$ be its sparsity pattern. If the
normalized matrix $\tilde{C}=D^{-1/2}C_{H}D^{-1/2}$ where $D=\mathrm{diag}(C_{H})$
satisfies the following conditions: 
\begin{enumerate}
\item $\tilde{C}$ is positive definite, 
\item $\tilde{C}$ is sign-consistent, 
\item We have
\begin{equation}
\beta\left(G_{H},\|\tilde{C}\|_{\max}\right)\leq\min_{(k,l)\notin G_{H}}\frac{\lambda_{k,l}-|(C_{H})_{k,l}|}{\sqrt{(C_{H})_{k,k}\cdot(C_{H})_{l,l}}}\label{eqbeta}
\end{equation}
\end{enumerate}
Then $C_{H}$ has the same sparsity pattern and opposite signs as
$\hat{X}$ in \eqref{eq:rgl}, i.e.
\begin{gather*}
(C_{H})_{i,j}=0\qquad\iff\qquad\hat{X}_{i,j}=0,\\
(C_{H})_{i,j}>0\qquad\iff\qquad\hat{X}_{i,j}<0,\\
(C_{H})_{i,j}<0\qquad\iff\qquad\hat{X}_{i,j}>0.
\end{gather*}

\end{restatable}
\begin{proof}
See Appendix~\ref{sec:appendRGL}.
\end{proof}
Theorem~\ref{thm:tt1} leads to the following corollary, which asserts
that the optimal solution of RGL can be obtained by\emph{ maximum
determinant matrix completion}: computing the matrix $Z\succeq0$
with the largest determinant that ``fills-in'' the zero elements
of $P_{H}(C_{\lambda})$. 
\begin{cor}
\label{cor1} Suppose that the conditions in Theorem~\ref{thm:tt1}
are satisfied. Define $\hat{Z}$ as the solution to the following
\begin{align}
\hat{Z}=\underset{Z\succeq0}{\text{ maximize }} & \log\det Z\label{eq:mdmc2}\\
\text{subject to } & Z_{i,j}=P_{H}(C_{\lambda})\text{ for all }(i,j)\nonumber \\
 & \quad\text{ where }[P_{H}(C_{\lambda})]_{i,j}\ne0\nonumber 
\end{align}
Then $\hat{Z}=\hat{X}^{-1}$, where $\hat{X}$ is the solution of
(\ref{eq:rgl}). 
\end{cor}
\begin{proof}
Under the conditions of Theorem~\ref{thm:tt1}, the $(i,j)$-th element
of the solution $\hat{X}_{i,j}$ has \emph{opposite signs} to the
corresponding element in $(C_{H})_{i,j}$, and hence also $C_{i,j}$.
Replacing each $|X_{i,j}|$ term in \eqref{eq:rgl} with $\mathrm{sign}(\hat{X}_{i,j})X_{i,j}=-\mathrm{sign}(C_{i,j})X_{i,j}$
yields
\begin{align}
\hat{X}=\underset{X\succ0}{\text{minimize }} & \underbrace{\tr CX-\sum_{i=1}^{n}\sum_{j=i+1}^{n}\mathrm{sign}(C_{i,j})\lambda_{i,j}X_{i,j}}_{\equiv\tr C_{\lambda}X}-\log\det X\label{eq:rgl2}\\
\text{subject to } & X_{i,j}=0\qquad\forall(i,j)\notin H.\nonumber 
\end{align}
The constraint $X\in\S_{H}^{n}$ further makes $\tr C_{\lambda}X=\tr C_{\lambda}P_{H}(X)=\tr P_{H}(C_{\lambda})X\equiv\tr C_{H}X$.
Taking the dual of~\eqref{eq:rgl2} yields~\eqref{eq:mdmc2}; complementary
slackness yields $\hat{Z}=\hat{X}^{-1}$.
\end{proof}
Standard manipulations show that (\ref{eq:mdmc2}) is the Lagrangian
dual of (\ref{eq:mdmc1}), thus explaining the etymology of (\ref{eq:mdmc1})
as MDMC.

\section{\label{sec:Proposed-Algorithm}Proposed Algorithm}

This section describes an efficient algorithm to solve MDMC (\ref{eq:mdmc1})
in which the sparsity pattern $G$ is \emph{nonchordal}. If we assume
that the input matrix $C_{\lambda}$ is sparse, and that sparse Cholesky
factorization is able to solve $C_{\lambda}x=b$ in $O(n)$ time,
then our algorithm is guaranteed to compute an $\epsilon$-accurate
solution in $O(n\log\epsilon^{-1})$ time and $O(n)$ memory.

The algorithm is fundamentally a Newton-CG method, i.e. Newton's method
in which the Newton search directions are computed using conjugate
gradients (CG). It is developed from four key insights:

\textbf{1. Chordal embedding is easy via sparse matrix heuristics.}
State-of-the-art algorithms for (\ref{eq:mdmc1}) begin by computing
a chordal embedding $\tilde{G}$ for $G$. The optimal chordal embedding
with the fewest number of nonzeros $|\tilde{G}|$ is NP-hard to compute,
but a good-enough embedding with $O(n)$ nonzeros is sufficient for
our purposes. Computing a good $\tilde{G}$ with $|\tilde{G}|=O(n)$
is exactly the same problem as finding a sparse Cholesky factorization
$C_{\lambda}=LL^{T}$ with $O(n)$ fill-in. Using heuristics developed
for numerical linear algebra, we are able to find sparse chordal embeddings
for graphs containing millions of edges and hundreds of thousands
of nodes in seconds. 

\textbf{2. Optimize directly on the sparse matrix cone.} Using log-det
barriers for sparse matrix cones~\cite{dahl2008covariance,Andersen2010,andersen2013logarithmic,vandenberghe2015chordal},
we can optimize directly in the space $\S_{\tilde{G}}^{n}$, while
ignoring all matrix elements outside of $\tilde{G}$. If $|\tilde{G}|=O(n)$,
then only $O(n)$ decision variables must be explicitly optimized.
Moreover, each function evaluation, gradient evaluation, and matrix-vector
product with the Hessian can be performed in $O(n)$ time, using the
numerical recipes in~\cite{andersen2013logarithmic}.

\textbf{3. The dual is easier to solve than the primal.} The primal
problem starts with a feasible point $X\in\S_{\tilde{G}}^{n}$ and
seeks to achieve first-order optimality. The dual problem starts with
an infeasible optimal point $X\notin\S_{\tilde{G}}^{n}$ satisfying
first-order optimality, and seeks to make it feasible. Feasibility
is easier to achieve than optimality, so the dual problem is easier
to solve than the primal.

\textbf{4. Conjugate gradients (CG) converges in $O(1)$ iterations.}
Under the same conditions that allow Theorem~\ref{thm:tt1} to work,
our main result (Theorem~\ref{thm:cond_bound}) bounds the condition
number of the Newton subproblem to be $O(1)$, independent of the
problem dimension $n$ and the current accuracy $\epsilon$. It is
therefore cheaper to solve this subproblem using CG \emph{to machine
precision} $\delta_{\mach}$ in $O(n\log\delta_{\mach}^{-1})$ time
than it is to solve for it directly in $O(nm^{2}+m^{3})$ time using
Cholesky factorization~\cite{dahl2008covariance,Andersen2010,andersen2013logarithmic}.
Moreover, CG is an optimal Krylov subspace method, and as such, it
is often able to exploit clustering in the eigenvalues to converge
superlinearly. Finally, computing the Newton direction to high accuracy
further allows the outer Newton method to also converge quadratically.

The remainder of this section describes each consideration in further
detail. We state the algorithm explicitly in Section~\ref{subsec:The-full-algorithm}.

\subsection{Efficient chordal embedding}

Following~\cite{dahl2008covariance}, we begin by reformulating (\ref{eq:mdmc1})
into a sparse chordal matrix program 
\begin{align}
\hat{X}=\text{ minimize } & \tr CX-\log\det X\label{eq:prim1}\\
\text{subject to } & X_{i,j}=0\quad\forall(i,j)\in\tilde{G}\backslash G.\nonumber \\
 & X\in\S_{\tilde{G}}^{n}.\nonumber 
\end{align}
in which $\tilde{G}$ is a \emph{chordal embedding} for $G$: a sparsity
pattern $\tilde{G}\supset G$ whose graph contains no induced cycles
greater than three. This can be implemented using standard algorithms
for large-and-sparse linear equations, due to the following result. 
\begin{prop}
\label{prop:embedding}Let $C\in\S_{G}^{n}$ be a positive definite
matrix with sparsity pattern $G$. Compute its unique lower-triangular
Cholesky factor $L$ satisfying $C=LL^{T}$. Ignoring perfect numerical
cancellation, the sparsity pattern of $L+L^{T}$ is a chordal embedding
$\tilde{G}\supset G$. 
\end{prop}
\begin{proof}
The original proof is due to~\cite{rose1970triangulated}; see also~\cite{vandenberghe2015chordal}.
\end{proof}
Note that $\tilde{G}$ can be determined directly from $G$ using
a \emph{symbolic} Cholesky algorithm, which simulates the steps of
Gaussian elimination using Boolean logic. Moreover, we can substantially
reduce the number of edges added to $G$ by reordering the columns
and rows of $C$ using a \emph{fill-reducing ordering}. 
\begin{cor}
Let $\Pi$ be a permutation matrix. For the same $C\in\S_{G}^{n}$
in Proposition~\ref{prop:embedding}, compute the unique Cholesky
factor satisfying $\Pi C\Pi^{T}=LL^{T}$. Ignoring perfect numerical
cancellation, the sparsity pattern of $\Pi(L+L^{T})\Pi^{T}$ is a
chordal embedding $\tilde{G}\supset G$. 
\end{cor}
The problem of finding the best choice of $\Pi$ is known as the \emph{fill-minimizing}
problem, and is NP-complete~\cite{yannakakis1981computing}. However,
good orderings are easily found using heuristics developed for numerical
linear algebra, like minimum degree ordering~\cite{george1989evolution}
and nested dissection~\cite{gilbert1988some,agrawal1993cutting}.
In fact, \cite{gilbert1988some}~proved that nested dissection is
$O(\log(n))$ suboptimal for bounded-degree graphs, and notes that
``we do not know a class of graphs for which {[}nested dissection
is suboptimal{]} by more than a constant factor.\textquotedblright{}

If $G$ admits sparse chordal embeddings, then a good-enough $|\tilde{G}|=O(n)$
will usually be found using minimum degree or nested dissection. In
MATLAB, the minimum degree ordering and symbolic factorization steps
can be performed in two lines of code; see the snippet in Figure~\ref{fig:MATLAB-code-snippet}.
\begin{figure}
\begin{lstlisting}
p = amd(C); % fill-reducing ordering
[~,~,~,~,R] = symbfact(C(p,p)); % chordal embedding by elimination
Gt = R+R'; Gt(p,p) = Gt; % recover embedded pattern
m = nnz(R)-nnz(tril(C)); % count the number of added eges
\end{lstlisting}

\caption{\label{fig:MATLAB-code-snippet}MATLAB code for chordal embedding
via its internal approximate minimum degree ordering. Given a sparse
matrix (\texttt{C}), compute a chordal embedding (\texttt{Gt}) and
the number of added edges (m).}
\end{figure}

\subsection{Logarithmic barriers for sparse matrix cones}

Define the cone of \emph{sparse positive semidefinite matrices} $\K$,
and the cone of \emph{sparse matrices with positive semidefinite completions}
$\K_{*}$, as the following 
\[
\K=\S_{+}^{n}\cap\S_{\tilde{G}}^{n},\qquad\K_{*}=\{S\bullet X\ge0:S\in\S_{\tilde{G}}\}=P_{\tilde{G}}(\S_{+}^{n}).
\]
Then (\ref{eq:prim1}) can be posed as the primal-dual pair: 
\begin{align}
\arg\min_{X\in\K} & \{C\bullet X+f(X):A^{T}(X)=0\},\label{eq:prim2}\\
\arg\max_{S\in\K_{*},y\in\mathbb{R}^{m}} & \{-f_{*}(S):S=C-A(y)\},\label{eq:dual2}
\end{align}
with in which $f$ and $f_{*}$ are the ``log-det'' barrier functions
on $\K$ and $\K_{*}$ as introduced by~\cite{dahl2008covariance,Andersen2010,andersen2013logarithmic}
\[
f(X)=-\log\det X,\qquad f_{*}(S)=-\min_{X\in\K}\{S\bullet X-\log\det X\}.
\]
The linear map $A:\R^{m}\to\S_{\tilde{G}\backslash G}^{n}$ converts
a list of $m$ variables into the corresponding matrix in $\tilde{G}\backslash G$.
The gradients of $f$ are simply the projections of their usual values
onto $\S_{\tilde{G}}^{n}$, as in 
\[
\nabla f(X)=-P_{\tilde{G}}(X^{-1}),\qquad\nabla^{2}f(X)[Y]=P_{\tilde{G}}(X^{-1}YX^{-1}).
\]
Given any $S\in\K_{*}$ let $X\in\K$ be the unique matrix satisfying
$P_{\tilde{G}}(X^{-1})=S$. Then we have 
\[
f_{*}(S)=n+\log\det X,\qquad\nabla f_{*}(S)=-X,\qquad\nabla^{2}f_{*}(S)[Y]=\nabla^{2}f(X)^{-1}[Y].
\]
Assuming that $\tilde{G}$ is \emph{sparse }and \emph{chordal}, all
six operations can be efficiently evaluated in $O(n)$ time and $O(n)$
memory, using the numerical recipes described in~\cite{andersen2013logarithmic}.

\subsection{Solving the dual problem}

Our algorithm actually solves the dual problem (\ref{eq:dual2}),
which can be rewritten as an unconstrained optimization problem 
\begin{equation}
\hat{y}\equiv\arg\min_{y\in\R^{m}}g(y)\equiv f_{*}(C_{\lambda}-A(y)).\label{eq:orig_prob}
\end{equation}
After the solution $\hat{y}$ is found, we can recover the optimal
estimator for the primal problem via $\hat{X}=-\nabla f_{*}(C_{\lambda}-A(y))$.
The dual problem (\ref{eq:dual2}) is easier to solve than the primal
(\ref{eq:prim2}) because the origin $y=0$ often lies very close
to the solution $\hat{y}$. To see this, note that $y=0$ produces
a candidate estimator $\tilde{X}=-\nabla f_{*}(C_{\lambda})$ that
solves the \emph{chordal} matrix completion problem
\[
\tilde{X}=\arg\min\{\tr C_{\lambda}X-\log\det X:X\in\S_{\tilde{G}}^{n}\},
\]
which is a relaxation of the nonchordal problem posed over $\S_{G}^{n}$.
As observed by several previous authors~\cite{dahl2008covariance},
this relaxation is a high quality guess, and $\tilde{X}$ is often
``almost feasible'' for the original nonchordal problem posed over
$\S_{G}^{n}$, as in $\tilde{X}\approx P_{G}(\tilde{X})$. Some simple
algebra shows that the gradient $\nabla g$ evaluated at the origin
has Euclidean norm $\|\nabla g(0)\|=\|\tilde{X}-P_{G}(\tilde{X})\|_{F}$,
so if $\tilde{X}\approx P_{G}(\tilde{X})$ holds true, then the origin
$y=0$ is close to optimal. Starting from this point, we can expect
Newton's method to rapidly converge at a quadratic rate.

\subsection{\label{sub:Proof-of-Linear}CG converges in $O(1)$ iterations}

The most computationally expensive part of Newton's method is the
solution of the Newton direction $\Delta y$ via the $m\times m$
system of equations 
\begin{equation}
\nabla^{2}g(y)\Delta y=-\nabla g(y).\label{eq:newton_sys}
\end{equation}
The Hessian matrix $\nabla^{2}g(y)$ is fully dense, but matrix-vector
products are linear $O(n)$ time using the algorithms in Section~\ref{subsec:Logarithmic-barriers-for}.
This insight motivates solving (\ref{eq:newton_sys}) using an iterative
Krylov subspace method like conjugate gradients (CG), which is a \emph{matrix-free}
method that requires a single matrix-vector product with $\nabla^{2}g(y)$
at each iteration~\cite{barrett1994templates}. Starting from the
origin $p=0$, the method converges to an $\epsilon$-accurate search
direction $p$ satisfying 
\[
(p-\Delta y)^{T}\nabla^{2}g(y)(p-\Delta y)\le\epsilon|\Delta y^{T}\nabla g(y)|
\]
in at most 
\begin{equation}
\left\lceil \sqrt{\kappa_{g}}\log(2/\epsilon)\right\rceil \text{ CG iterations,}\label{eq:cg_iter_bnd}
\end{equation}
where $\kappa_{g}=\|\nabla^{2}g(y)\|\|\nabla^{2}g(y)^{-1}\|$ is the
condition number of the Hessian matrix~\cite{greenbaum1997iterative,saad2003iterative}.
In many important convex optimization problems, the condition number
$\kappa_{g}$ grows like $O(1/\epsilon)$ or $O(1/\epsilon^{2})$
as the outer Newton iterates approach an $\epsilon$-neighborhood
of the true solution. As a consequence, Newton-CG methods typically
require $O(1/\sqrt{\epsilon})$ or $O(1/\epsilon)$ CG iterations.

It is therefore surprising that we are able to bound $\kappa_{g}$
globally for the MDMC problem. Below, we state our main result, which
says that the condition number $\kappa_{g}$ depends polynomially
on the problem data and the quality of the initial point, but is \emph{independent
of the problem dimension} $n$\emph{ and the accuracy of the current
iterate $\epsilon$. }

\begin{restatable}{thm}{condbound}\label{thm:cond_bound}At any $y$
satisfying $g(y)\le g(y_{0})$ and $\nabla g(y)^{T}(y-y_{0})\le\phi_{\max}$,
the condition number $\kappa_{g}$ of the Hessian matrix $\nabla^{2}g(y)$
is bound 
\begin{align}
\kappa_{g} & \le4\left(1+\frac{\phi_{\max}^{2}\lambda_{\max}(X_{0})}{\lambda_{\min}(\hat{X})}\right)^{2}\label{eq:thm_bnd1}
\end{align}
where:
\begin{itemize}
\item $\phi_{\max}=g(y_{0})-g(\hat{y})$ is the suboptimality of the initial
point,
\item $\A=[\vector A_{1},\ldots,\vector A_{m}]$ is the vectorized version
of the problem data,
\item $X_{0}=-\nabla f_{*}(S_{0})$ and $S_{0}=C-A(y_{0})$ are the initial
primal-dual pair,
\item $\hat{X}=-\nabla f_{*}(\hat{S})$ and $\hat{S}=C-A(\hat{y})$ are
the solution primal-dual pair.
\end{itemize}
\end{restatable}
\begin{proof}
See Appendix~\ref{sec:appenCond}.
\end{proof}
\begin{rem}
Newton's method is a descent method, so its $k$-th iterate $y_{k}$
trivially satisfies $g(y_{k})\le g(y_{0})$. Technically, the condition
$\nabla g(y_{k})^{T}(y_{k}-y_{0})\le\phi_{\max}$ can be guaranteed
by enclosing Newton's method within an outer auxillary path-following
loop; see Section 4.3.5 of~\cite{nesterov2013introductory}. In practice,
naive Newton's method will usually satisfy the condition on its own;
see our numerical experiments in Section~\ref{sec:Numerical-Results}.
\end{rem}
Applying Theorem~\ref{thm:cond_bound} to (\ref{eq:cg_iter_bnd})
shows that CG solves each Newton subproblem to $\epsilon$-accuracy
in $O(\log\epsilon^{-1})$ iterations. Multiplying this figure by
the $O(\log\log\epsilon^{-1})$ Newton steps to converge yields a
\emph{global} iteration bound of 
\[
O(\log\epsilon^{-1}\cdot\log\log\epsilon^{-1})\approx O(1)\text{ CG iterations.}
\]
Multiplying this figure by the $O(n)$ cost of each CG iteration proves
the claimed time complexity in (\ref{eq:lintime}). In practice, CG
typically converges much faster than this worst-case bound, due to
its ability to exploit the clustering of eigenvalues in $\nabla^{2}g(y)$;
see~\cite{greenbaum1997iterative,saad2003iterative}. Moreover, accurate
Newton directions are only needed to guarantee quadratic convergence
close to the solution. During the initial Newton steps, we may loosen
the error tolerance for CG for a significant speed-up. Inexact Newton
steps can be used to obtain a speed-up of a factor of 2-3.

\subsection{\label{subsec:The-full-algorithm}The full algorithm}

To summarize, we begin by computing a chordal embedding $\tilde{G}$
for the sparsity pattern $G$ of $C_{\lambda}$, using the code snippet
in Figure~\ref{fig:MATLAB-code-snippet}. We use the embedding to
reformulate (\ref{eq:mdmc1}) as (\ref{eq:prim1}), and solve the
unconstrained problem $\hat{y}=\min_{y}g(y)$ defined in (\ref{eq:orig_prob}),
using Newton's method
\begin{align*}
y_{k+1} & =y_{k}+\alpha_{k}\Delta y_{k}, & \Delta y_{k} & \equiv-\nabla^{2}g(y_{k})^{-1}\nabla g(y_{k})
\end{align*}
starting at the origin $y_{0}=0$. The function value $g(y)$, gradient
$\nabla g(y)$ and Hessian matrix-vector products are all evaluated
using the numerical recipes described by~\cite{andersen2013logarithmic}.

At each $k$-th Newton step, we compute the Newton search direction
$\Delta y_{k}$ using conjugate gradients. A loose tolerance is used
when the Newton decrement $\delta_{k}=|\Delta y_{k}^{T}\nabla g(y_{k})|$
is large, and a tight tolerance is used when the decrement is small,
implying that the iterate is close to the true solution. 

Once a Newton direction $\Delta y_{k}$ is computed with a sufficiently
large Newton decrement $\delta_{k}$, we use a backtracking line search
to determine the step-size $\alpha_{k}$. In other words, we select
the first instance of the sequence $\{1,\rho,\rho^{2},\rho^{3},\dots\}$
that satisfies the Armijo\textendash Goldstein condition 
\[
g(y+\alpha\Delta y)\le g(y)+\gamma\alpha\Delta y^{T}\nabla g(y),
\]
in which $\gamma\in(0,0.5)$ and $\rho\in(0,1)$ are line search parameters.
Our implementation used $\gamma=0.01$ and $\rho=0.5$. We complete
the step and repeat the process, until convergence.

We terminate the outer Newton's method if the Newton decrement $\delta_{k}$
falls below a threshold. This implies either that the solution has
been reached, or that CG is not converging to a good enough $\Delta y_{k}$
to make significant progress. The associated estimator for $\Sigma^{-1}$
is recovered by evaluating $\hat{X}=-\nabla f_{*}(C_{\lambda}-A(\hat{y}))$.

\section{\label{sec:Numerical-Results}Numerical Results}

\begin{figure}[t]
\hfill{}\subfloat[]{\includegraphics[width=0.48\columnwidth]{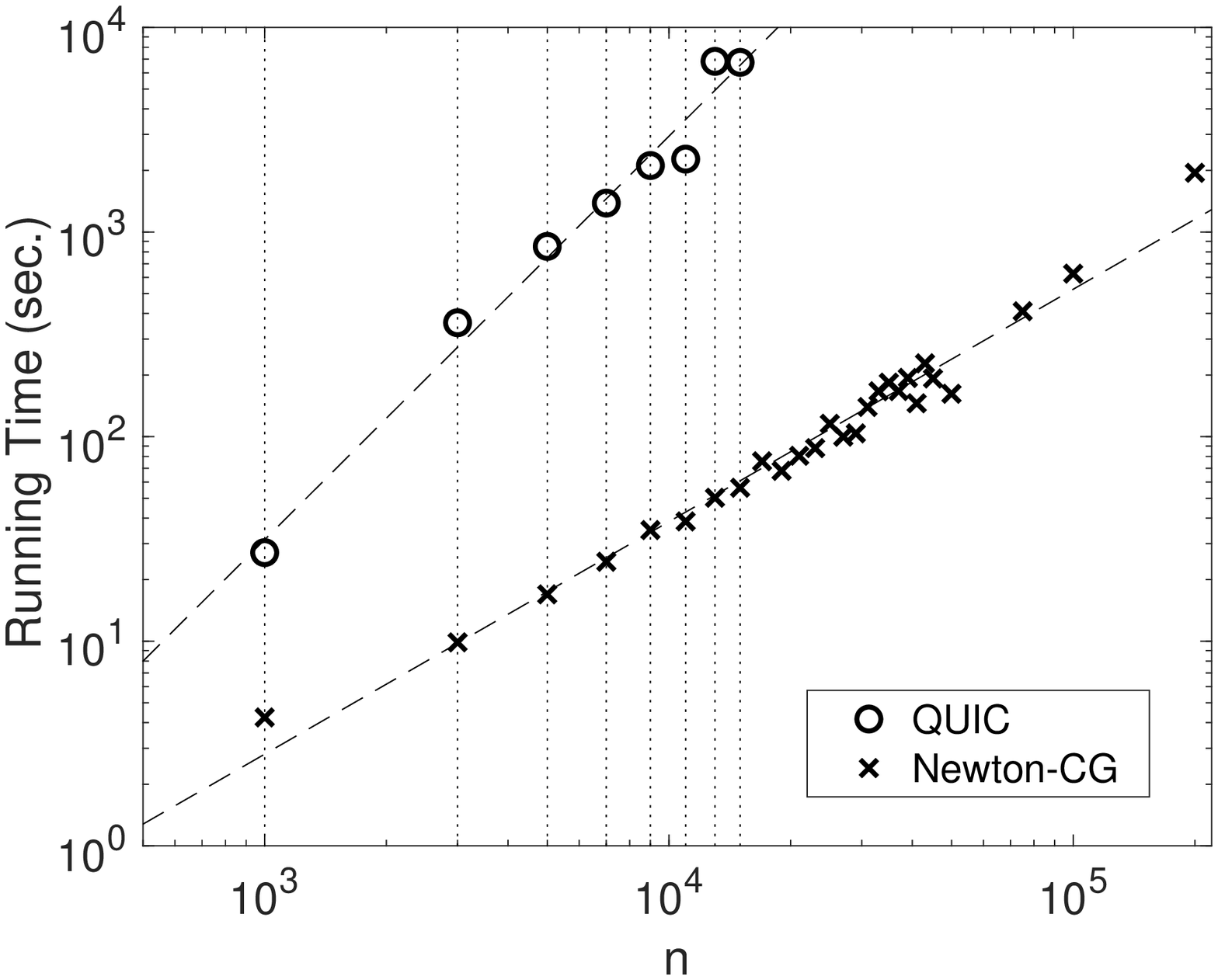}

}\hfill{}\subfloat[]{\includegraphics[width=0.48\columnwidth]{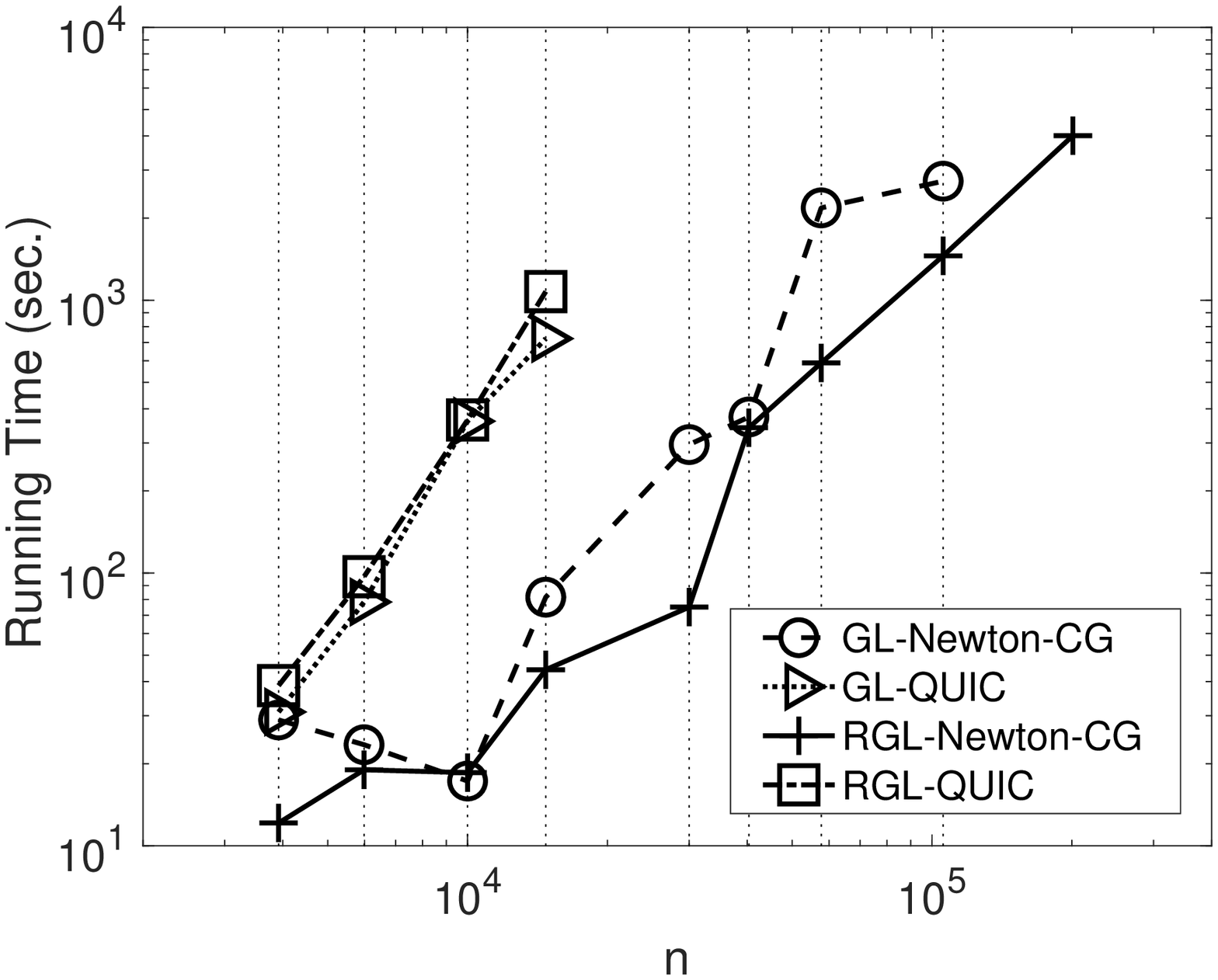}

}\hfill{} \caption{CPU time Newton-CG vs QUIC: (a) case study 1; (b) case study 2.}
\label{fig:results} 
\end{figure}
Finally, we benchmark our algorithm\footnote{The MATLAB source code for our solver can be found at \url{http://alum.mit.edu/www/ryz}}
against \texttt{QUIC}~\cite{Hsieh14}, commonly considered the fastest
solver for graphical lasso or RGL\footnote{\texttt{QUIC} was taken from \url{http://bigdata.ices.utexas.edu/software/1035/}}.
(Another widely-used algorithm is \texttt{GLASSO}~\cite{Friedman08},
but we found it to be significantly slower than \texttt{QUIC}.) We
consider two case studies. The first case study numerically verifies
the claimed $O(n)$ complexity of our MDMC algorithm on problems with
a nearly-banded structure. The second case study performs the full
threshold-MDMC procedure for graphical lasso and RGL, on graphs collected
from real-life applications. 

All experiments are performed on a laptop computer with an Intel Core
i7 quad-core 2.50 GHz CPU and 16GB RAM. The reported results are based
on a serial implementation in MATLAB-R2017b. Both our Newton decrement
threshold and \texttt{QUIC}'s convergence threshold are $10^{-7}$. 

We implemented the soft-thresholding set (\ref{thresh2}) as a serial
routine that uses $O(n)$ memory by taking the $n\times N$ matrix-of-samples
$\mathbf{X}=[\mathbf{x}_{1},\mathbf{x}_{2},\ldots,\mathbf{x}_{N}]$
satisfying $C=\frac{1}{N}\mathbf{X}\mathbf{X}^{T}$ as input. The
routine implicitly partitions $C$ into submatrices of size $4000\times4000$,
and iterates over the submatrices one at a time. For each submatrix,
it explicitly forms the submatrix, thresholds it using dense linear
algebra, and then stores the result as a sparse matrix. 

\subsection{Case Study 1: Banded Patterns}

The first case study aims to verify the claimed $O(n)$ complexity
of our algorithm for MDMC. Here, we avoid the proposed thresholding
step, and focus solely on the MDMC (\ref{eq:mdmc1}) problem. Each
sparsity pattern $G$ is a corrupted banded matrices with bandwidth
$101$. The off-diagonal nonzero elements of $C$ are selected from
the uniform distribution in $[-2,0)$ and then corrupted to zero with
probability 0.3. The diagonal elements are fixed to $5$. Our numerical
experiments fix the bandwidth and vary the number of variables $n$
from 1,000 to 200,000. A time limit of 2 hours is set for both algorithms.

Figure~\ref{fig:results}a compares the running time of both algorithms.
A log-log regression results in an empirical time complexity of $O(n^{1.1})$
for our algorithm, and $O(n^{2})$ for \texttt{QUIC}. The extra 0.1
in the exponent is most likely an artifact our MATLAB implementation.
In either case, \texttt{QUIC}'s quadratic complexity limits it to
$n=1.5\times10^{4}$. By contrast, our algorithm solves an instance
with $n=2\times10^{5}$ in less than 33 minutes. The resulting solutions
are extremely accurate, with optimality and feasibility gaps of less
than $10^{-16}$ and $10^{-7}$, respectively.

\subsection{Case Study 2: Real-Life Graphs}

\begin{table*}
\scriptsize

\hfill{}%
\begin{tabular}{|c|c|c|c|c|c||c|c|c||c||c|c|}
\hline 
\multicolumn{6}{|c|||}{} & \multicolumn{3}{c||}{Newton-CG} & \multicolumn{1}{c|||}{QUIC} & \multicolumn{2}{c|}{}\tabularnewline
\hline 
\#  & file name  & type  & $n$  & $m$ & $m/n$ & sec  & gap  & feas  & sec  & diff. gap  & speed-up\tabularnewline
\hline 
\hline 
1  & freeFlyingRobot-7  & GL  & 3918  & 20196 & 5.15 & 28.9  & 5.7e-17  & 2.3e-7  & 31.0  & 3.9e-4  & 1.07\tabularnewline
\hline 
1  & freeFlyingRobot-7  & RGL  & 3918  & 20196 & 5.15 & 12.1  & 6.5e-17  & 2.9e-8  & 38.7  & 3.8e-5  & 3.20 \tabularnewline
\hline 
2  & freeFlyingRobot-14  & GL  & 5985  & 27185 & 4.56 & 23.5  & 5.4e-17  & 1.1e-7  & 78.3  & 3.8e-4  & 3.33 \tabularnewline
\hline 
2  & freeFlyingRobot-14  & RGL  & 5985  & 27185 & 4.56 & 19.0  & 6.0e-17  & 1.7e-8  & 97.0  & 3.8e-5  & 5.11 \tabularnewline
\hline 
3  & cryg10000  & GL  & 10000  & 170113 & 17.0 & 17.3  & 5.9e-17  & 5.2e-9  & 360.3  & 1.5e-3  & 20.83 \tabularnewline
\hline 
3  & cryg10000  & RGL  & 10000  & 170113 & 17.0 & 18.5  & 6.3e-17  & 1.0e-7  & 364.1  & 1.9e-5  & 19.68 \tabularnewline
\hline 
4  & epb1  & GL  & 14734  & 264832 & 18.0 & 81.6  & 5.6e-17  & 4.3e-8  & 723.5  & 5.1e-4  & 8.86 \tabularnewline
\hline 
4  & epb1  & RGL  & 14734  & 264832 & 18.0 & 44.2  & 6.2e-17  & 3.3e-8  & 1076.4  & 4.2e-4  & 24.35 \tabularnewline
\hline 
5  & bloweya  & GL  & 30004  & 10001 & 0.33 & 295.8  & 5.6e-17  & 9.4e-9  & $*$  & $*$  & $*$ \tabularnewline
\hline 
5  & bloweya  & RGL  & 30004  & 10001 & 0.33 & 75.0  & 5.5e-17  & 3.6e-9  & $*$  & $*$  & $*$ \tabularnewline
\hline 
6  & juba40k  & GL  & 40337  & 18123 & 0.44 & 373.3  & 5.6e-17  & 2.6e-9  & $*$  & $*$  & $*$ \tabularnewline
\hline 
6  & juba40k  & RGL  & 40337  & 18123 & 0.44 & 341.1  & 5.9e-17  & 2.7e-7  & $*$  & $*$  & $*$ \tabularnewline
\hline 
7  & bayer01  & GL  & 57735  & 671293 & 11.6 & 2181.3  & 5.7e-17  & 5.2e-9  & $*$  & $*$  & $*$ \tabularnewline
\hline 
7  & bayer01  & RGL  & 57735  & 671293 & 11.6 & 589.1  & 6.4e-17  & 1.0e-7  & $*$  & $*$  & $*$ \tabularnewline
\hline 
8  & hcircuit  & GL  & 105676  & 58906 & 0.55 & 2732.6  & 5.8e-17  & 9.0e-9  & $*$  & $*$  & $*$ \tabularnewline
\hline 
8  & hcircuit  & RGL  & 105676  & 58906 & 0.55 & 1454.9  & 6.3e-17  & 7.3e-8  & $*$  & $*$  & $*$ \tabularnewline
\hline 
9  & co2010  & RGL  & 201062  & 1022633 & 5.08 & 4012.5  & 6.3e-17  & 4.6e-8  & $*$  & $*$  & $*$ \tabularnewline
\hline 
\end{tabular}\hfill{}

\caption{Details of case study 2. Here, ``$n$'' is the size of the covariance
matrix, ``$m$'' is the number of edges added to make its sparsity
graph chordal, ``sec'' is the running time in seconds, ``gap''
is the optimality gap, ``feas'' is the feasibility the solution,
``diff. gap'' is the difference in duality gaps for the two different
methods, and ``speed-up'' is the fact speed-up over QUIC achieved
by our algorithm.}
\label{Table}
\end{table*}
The second case study aims to benchmark the full thresholding-MDMC
procedure for sparse inverse covariance estimation on real-life graphs.
The actual graphs (i.e. the sparsity patterns) for $\Sigma^{-1}$
are chosen from \textit{SuiteSparse Matrix Collection}~\cite{Sparse11}\textemdash a
publicly available dataset for large-and-sparse matrices collected
from real-world applications. Our chosen graphs vary in size from
$n=3918$ to $n=201062$, and are taken from applications in chemical
processes, material science, graph problems, optimal control and model
reduction, thermal processes and circuit simulations.

For each sparsity pattern $G$, we design a corresponding $\Sigma^{-1}$
as follows. For each $(i,j)\in G$, we select $(\Sigma^{-1})_{i,j}=(\Sigma^{-1})_{j,i}$
from the uniform distribution in $[-1,1]$, and then corrupt it to
zero with probability $0.3$. Then, we set each diagonal to $(\Sigma^{-1})_{i,i}=1+\sum_{j}|(\Sigma^{-1})_{i,j}|$.
Using this $\Sigma$, we generate $N=5000$ samples i.i.d. as $\mathbf{x}_{1},\ldots,\mathbf{x}_{N}\sim\mathcal{N}(0,\Sigma)$.
This results in a sample covariance matrix $C=\frac{1}{N}\sum_{i=1}^{N}\mathbf{x}_{i}\mathbf{x}_{i}^{T}$.

We solve graphical lasso and RGL with the $C$ described above using
our proposed soft-thresholding-MDMC algorithm and \texttt{QUIC}, in
order to estimate $\Sigma^{-1}$. In the case of RGL, we assume that
the graph $G$ is known \emph{a priori}, while noting that 30\% of
the elements of $\Sigma^{-1}$ have been corrupted to zero. Our goal
here is to discover the location of these corrupted elements. In all
of our simulations, the threshold $\lambda$ is set so that the number
of nonzero elements in the the estimator is roughly the same as the
ground truth. We limit both algorithms to 3 hours of CPU time. 

Figure~\ref{fig:results}b compares the CPU time of both two algorithms
for this case study; the specific details are provided in Table~\ref{Table}.
A log-log regression results in an empirical time complexity of $O(n^{1.64})$
and $O(n^{1.55})$ for graphical lasso and RGL using our algorithm,
and $O(n^{2.46})$ and $O(n^{2.52})$ for the same using \texttt{QUIC}.
The exponents of our algorithm are $\ge1$ due to the initial soft-thresholding
step, which is quadratic-time on a serial computer, but $\le2$ because
the overall procedure is dominated by the solution of the MDMC. Both
algorithms solve graphs with $n\le1.5\times10^{4}$ within the allotted
time limit, though our algorithm is 11 times faster on average. Only
our algorithm is able to solve the estimation problem with $n\approx2\times10^{5}$
in a little more than an hour. 

To check whether thresholding-MDMC really does solve graphical lasso
and RGL, we substitute the two sets of estimators back into their
original problems (\ref{eq:gl}) and (\ref{eq:rgl}). The corresponding
objective values have a relative difference $\le4\times10^{-4}$,
suggesting that both sets of estimators are about equally optimal.
This observation verifies our claims in Theorem~\ref{thm:tt1} and
Corollary~\ref{cor1} that (\ref{eq:gl}) and (\ref{eq:rgl}): thresholding-MDMC
does indeed solve graphical lasso and RGL.

\section{Conclusions}

Graphical lasso is a widely-used approach for estimating a covariance
matrix with a sparse inverse from limited samples. In this paper,
we consider a slightly more general formulation called \emph{restricted}
graphical lasso (RGL), which additionally enforces a prior sparsity
pattern to the estimation. We describe an efficient approach that
substantially reduces the cost of solving RGL: 1) soft-thresholding
the sample covariance matrix and projecting onto the prior pattern,
to recover the estimator's sparsity pattern; and 2) solving a maximum
determinant matrix completion (MDMC) problem, to recover the estimator's
numerical values. The first step is quadratic $O(n^{2})$ time and
memory but embarrassingly parallelizable. If the resulting sparsity
pattern is \emph{sparse} and \emph{chordal}, then under mild technical
assumptions, the second step can be performed using the Newton-CG
algorithm described in this paper in linear $O(n)$ time and memory.
The algorithm is tested on both synthetic and real-life data, solving
instances with as many as 200,000 variables to 7-9 digits of accuracy
within an hour on a standard laptop computer.

\bibliographystyle{plain}
\bibliography{ncmc}

\global\long\def\A{\mathbf{A}}
\global\long\def\K{\mathcal{K}}
\global\long\def\R{\mathbb{R}}
\global\long\def\S{\mathbb{S}}
\global\long\def\dom{\mathrm{dom}\,}
\global\long\def\cond{\mathrm{cond}\,}
\global\long\def\vector{\mathrm{vec}\,}
\global\long\def\tr{\mathrm{tr}\,}

\appendix

\section{\label{sec:appendRGL}Restricted graphical Lasso and MDMC}

\global\long\def\tr{\mathrm{tr}\,}
Our aim is to elucidate the connection between the RGL and MDMC problem
under the assumption that the regularization coefficients are chosen
to be large, i.e., when a \emph{sparse} solution for the RGL is sought.
Recall that RGL is formulated as follows \begin{subequations}\label{RGL}
\begin{align}
{\text{minimize }} & \tr CX-\log\det X+\sum_{(i,j)\in V}\lambda_{ij}|X_{i,j}|\\
\text{s.t.}\  & \ X\in\mathbb{S}_{V}^{n}\label{sparsity}\\
 & \ X\succ0
\end{align}
\end{subequations}

Now, consider the following modified soft-thresholded sample covariance
matrix 
\begin{equation}
(C_{\lambda})_{i,j}=\begin{cases}
C_{i,j} & i=j\\
0 & i\not=j\ \text{and}\ (i,j)\not\in V\\
0 & i\not=j\ \text{and}\ (i,j)\in V\ \text{and}\ -\lambda_{i,j}\le,C_{i,j}\le\lambda_{i,j}\\
C_{i,j}-\lambda_{i,j}\ \mathrm{sign}(C_{i,j}) & i\not=j\ \text{and}\ (i,j)\in V\ \text{and}\ |C_{i,j}|>\lambda_{i,j}
\end{cases}\label{thresh2-1}
\end{equation}

\begin{defn}[\cite{Salar17}]
\label{def:dd1-1} For a given symmetric matrix $M$, let $V_{M}$
denote the minimal set such that $M\in\mathbb{S}_{V_{M}}^{n}$. $M$
is called \textbf{inverse-consistent} if there exists a matrix $N$
with zero diagonal such that \begin{subequations} 
\begin{align}
 & M+N\succ0\\
 & N_{i,j}=0\quad\forall(i,j)\notin V_{M}\\
 & (M+N)^{-1}\in\mathbb{S}_{V_{M}}^{n}
\end{align}
\end{subequations} The matrix $N$ is called an \textbf{inverse-consistent
complement } of $M$ and is denoted by $M^{(c)}$. Furthermore, $M$
is called \textbf{sign-consistent} if for every $(i,j)\in V_{M}$,
the $(i,j)$ entries of $M$ and $(M+M^{(c)})^{-1}$ have opposite
signs. 
\end{defn}
\begin{defn}[\cite{Salar17}]
Given a sparsity pattern $V$ and a scalar $\alpha$, define $\beta(V,\alpha)$
as the maximum of $\|M^{(c)}\|_{\max}$ over all inverse-consistent
positive-definite matrices $M$ with the diagonal entries all equal
to 1 such that $M\in\mathbb{S}_{V}^{n}$ and $\|M\|_{\max}\leq\alpha$. 
\end{defn}
Without loss of generality, we make the following mild assumption.
\begin{assume}\label{ass1} $\lambda_{i,j}\not=|C_{i,j}|$ and $\lambda_{i,j}>0$
for every $(i,j)\in V$. \end{assume}

We are now ready to restate Theorem~\ref{thm:tt1}.

\rglthm*

First, note that the diagonal elements of $\tilde{C}_{\lambda}$ are
$1$ and its off-diagonal elements are between $-1$ and $1$. A sparse
solution for RGL requires large regularization coefficients. This
leads to numerous zero elements in $\tilde{C}_{\lambda}$ and forces
the magnitude of the nonzero elements to be small. This means that,
in most instances, $\tilde{C}_{\lambda}$ is positive definite or
even diagonally dominant. Certifying Condition (ii) is hard in general.
However,~\cite{Salar17} shows that this condition is automatically
implied by Condition (i) when $V_{C_{\lambda}}$ induces an acyclic
structure. More generally,~\cite{Sojoudi16} shows that $\tilde{C}_{\lambda}$
is sign-consistent if $(\tilde{C}_{\lambda}+\tilde{C}_{\lambda}^{(c)})^{-1}$
is close to its first order Taylor expansion. This assumption holds
in practice due to the fact that the magnitude of the off-diagonal
elements of $\tilde{C}_{\lambda}+\tilde{C}_{\lambda}^{(c)}$ is small.
Furthermore,~\cite{Salar18} proves that this condition is necessary
for the equivalence between the sparsity patterns of thresholding
and GL when the regularization matrix is large enough. Finally,~\cite{Salar17}
shows that the left hand side of~(\ref{eqbeta}) is upper bounded
by $c\times\|\tilde{C}_{\lambda}\|_{\max}^{2}$ for some $c>0$ which
only depends on $V_{C_{\lambda}}$. This implies that, when $\|\tilde{C}_{\lambda}\|_{\max}$
is small, or equivalently the regularization matrix is large, Condition
(iii) is automatically satisfied.

\subsection{Proofs}

In this section, we present the technical proofs of our theorems.
To prove Theorem~\ref{thm:tt1}, we need a number of lemmas. First,
consider the RGL problem, with $\mathbb{S}_{V}^{n}=\mathbb{S}^{n}$.
The first lemma offers optimality (KKT) conditions for the unique
solution of this problem. 
\begin{lem}
\label{expsol} $X^{*}$ is the optimal solution of RGL problem with
$\mathbb{S}_{V}^{n}=\mathbb{S}^{n}$ if and only if it satisfies the
following conditions for every $i,j\in\{1,2,...,n\}$: \begin{subequations}\label{optcon}
\begin{align}
 & (X^{*})_{i,j}^{-1}=C_{i,j} &  & \text{if}\quad i=j\\
 & (X^{*})_{i,j}^{-1}=C_{i,j}+\lambda_{i,j}\times\text{{\rm sign}}(X_{i,j}^{*}) &  & \text{if}\quad X_{i,j}^{*}\not=0\\
 & C_{i,j}-\lambda_{i,j}\leq(X^{*})_{i,j}^{-1}\leq\Sigma_{i,j}+\lambda_{i,j} &  & \text{if}\quad X_{i,j}^{*}=0
\end{align}
\end{subequations} where $(X^{*})_{i,j}^{-1}$ denotes the $(i,j)^{\text{th}}$
entry of $(X^{*})^{-1}$. 
\end{lem}
\begin{proof}
The proof is straightforward and omitted for brevity. 
\end{proof}
Now, consider the following optimization: 
\begin{equation}
\min_{{X}\in\mathbb{S}_{+}^{n}}-\log\det({X})+\mathrm{trace}(\tilde{C}{X})+\sum_{(i,j)\in V}\tilde{\lambda}_{i,j}|{X}_{i,j}|+2\max_{k}\{{C}_{k,k}\}\sum_{(i,j)\in V^{(c)}}|{X}_{i,j}|\label{RGL2}
\end{equation}
where 
\begin{align}
\tilde{C}_{i,j}=\frac{C_{i,j}}{\sqrt{C_{i,i}\times C_{j,j}}}\quad\tilde{\lambda}_{i,j}=\frac{\lambda_{i,j}}{\sqrt{C_{i,i}\times C_{j,j}}}
\end{align}
Let $\tilde{X}$ denotes the optimal solution of~\eqref{RGL2}. Furthermore,
define $D$ as a diagonal matrix with $D_{i,i}=C_{i,i}$ for every
$i\in\{1,2,...,n\}$. The following lemma relates $\tilde{X}$ to
$X^{*}$. 
\begin{lem}
\label{l3} We have $X^{*}=D^{-1/2}\times\tilde{X}\times D^{-1/2}$. 
\end{lem}
\begin{proof}
To prove this lemma, we define an intermediate optimization problem.
Consider 
\begin{equation}
\min_{{X}\in\mathbb{S}_{+}^{n}}f(X)=-\log\det({X})+\mathrm{trace}({\Sigma}{X})+\sum_{(i,j)\in V}{\lambda}_{i,j}|{X}_{i,j}|+2\max_{k}\{{C}_{kk}\}\sum_{(i,j)\in V^{(c)}}|{X}_{i,j}|\label{RGL3}
\end{equation}
Denote $X^{\sharp}$ as the optimal solution for~\eqref{RGL3}. First,
we show that $X^{\sharp}=X^{*}$. Trivially, $X^{*}$ is a feasible
solution for~\eqref{RGL3} and hence $f(X^{\sharp})\leq f(X^{*})$.
Now, we prove that $X^{\sharp}$ is a feasible solution for~\eqref{RGL}.
To this goal, we show that $X_{ij}^{\sharp}=0$ for every $(i,j)\in V^{(c)}$.
By contradiction, suppose $X_{ij}^{\sharp}\not=0$ for some $(i,j)\in V^{(c)}$.
Note that, due to the positive definiteness of ${X^{\sharp}}^{-1}$,
we have 
\begin{equation}
(X^{\sharp})_{i,i}^{-1}\times(X^{\sharp})_{j,j}^{-1}-((X^{\sharp})_{i,j}^{-1})^{2}>0\label{pdsub}
\end{equation}
Now, based on Lemma~\ref{expsol}, one can write 
\begin{equation}
(X^{\sharp})_{ij}^{-1}=C_{i,j}+2\max_{k}\{C_{k,k}\}\times\mathrm{sign}(X_{i,j}^{\sharp})\label{hatij}
\end{equation}
Considering the fact that $C\succeq0$, we have $|C_{i,j}|\leq\max_{k}\{C_{k,k}\}$.
Together with~\eqref{hatij}, this implies that $|(X^{\sharp})_{i,j}^{-1}|\geq\max_{k}\{C_{k,k}\}$.
Furthermore, due to Lemma~\ref{expsol}, one can write $(X^{\sharp})_{i,i}^{-1}=C_{i,i}$
and $(X^{\sharp})_{j,j}^{-1}=C_{jj}$. This leads to 
\begin{equation}
(X^{\sharp})_{i,i}^{-1}\times(X^{\sharp})_{j,j}^{-1}-((X^{\sharp})_{i,j}^{-1})^{2}=C_{i,i}\times C_{j,j}-(\max_{k}\{C_{k,k}))^{2}\leq0
\end{equation}
which contradicts with~\eqref{pdsub}. Therefore, $X^{\sharp}$ is
a feasible solution for~\eqref{RGL}. This implies that $f(X^{\sharp})\geq f(X^{*})$
and hence, $f(X^{*})=f(X^{\sharp})$. Due to the uniqueness of the
solution of~\eqref{RGL3}, we have $X^{*}=X^{\sharp}$. Now, note
that~\eqref{RGL3} can be reformulated as 
\begin{align}
\min_{{X}\in\mathbb{S}_{+}^{n}}-\log\det({X})+\mathrm{trace}(\tilde{C}D^{1/2}{X}D^{1/2})+\sum_{(i,j)\in V}{\lambda}_{i,j}|{X}_{i,j}|+2\max_{k}\{{C}_{k,k}\}\sum_{(i,j)\in V^{(c)}}|{X}_{ij}|
\end{align}
Upon defining 
\begin{equation}
\tilde{X}=D^{1/2}XD^{1/2}\label{Stilde}
\end{equation}
and following some algebra, one can verify that~\ref{RGL3} is equivalent
to 
\begin{equation}
\min_{\tilde{X}\in\mathbb{S}_{+}^{n}}-\log\det(\tilde{X})+\mathrm{trace}(\tilde{C}\tilde{X})+\sum_{(i,j)\in V}\tilde{\lambda}_{i,j}|\tilde{X}_{i,j}|+2\max_{k}\{\tilde{C}_{k,k}\}\sum_{(i,j)\in V^{(c)}}|\tilde{X}_{i,j}|+\log\det(D)\label{RGL4}
\end{equation}
Dropping the constant term in~\eqref{RGL4} gives rise to the optimization~\eqref{RGL2}.
Therefore, $X^{*}=D^{-1/2}\times\tilde{X}\times D^{-1/2}$ holds in
light of~\ref{Stilde}. This completes the proof. 
\end{proof}
Now, we present the proof of Theorem~\ref{thm:tt1}.\vspace*{2mm}

\textit{Proof of Theorem~\ref{thm:tt1}:} Note that, due to the definition
of $\tilde{C}_{\lambda}$ and Lemma~\ref{l3}, $\tilde{C}_{\lambda}$
and $\tilde{X}$ have the same signed sparsity pattern as ${C}_{\lambda}$
and ${X}^{*}$, respectively. Therefore, it suffices to show that
the signed sparsity structures of $\tilde{C}_{\lambda}$ and $\tilde{X}$
are the same.

To verify this, we focus on the optimality conditions for optimization~\eqref{RGL2}.
Due to Condition~(1-i), $\tilde{C}_{\lambda}$ is inverse-consistent
and has a unique inverse-consistent complement, which is denoted by
$N$. First, we will show that $(\tilde{C}_{\lambda}+N)^{-1}$ is
the optimal solution of~\eqref{RGL2}. For an arbitrary pair $(i,j)\in\{1,...,n\}$,
the KKT conditions, introduced in Lemma~\ref{expsol}, imply that
one of the following cases holds: 
\begin{itemize}
\item[1)] $i=j$: We have $(\tilde{C}_{\lambda}+N)_{i,j}=[\tilde{C}_{\lambda}]_{i,i}=\tilde{C}_{i,i}$. 
\item[2)] $(i,j)\in V_{C_{\lambda}}$: In this case, we have 
\begin{equation}
(\tilde{C}_{\lambda}+N)_{ij}=[\tilde{C}_{\lambda}]_{ij}=\tilde{C}_{ij}-\tilde{\lambda}_{ij}\times\mathrm{sign}(\tilde{C}_{ij})
\end{equation}
Note that since $|\tilde{C}_{ij}|>\tilde{\lambda}_{ij}$, we have
that $\mathrm{sign}([\tilde{C}_{\lambda}]_{ij})=\mathrm{sign}(\tilde{C}_{ij})$.
On the other hand, due to the sign-consistency of $\tilde{C}_{\lambda}$,
we have $\mathrm{sign}([\tilde{C}_{\lambda}]_{ij})=-\mathrm{sign}\left(\left((\tilde{C}_{\lambda}+N)^{-1}\right)_{ij}\right)$.
This implies that 
\begin{equation}
(\tilde{C}_{\lambda}+N)_{ij}=\tilde{C}_{ij}+\tilde{\lambda}_{ij}\times\mathrm{sign}\left(\left((\tilde{C}_{\lambda}+N)^{-1}\right)_{ij}\right)
\end{equation}
\item[3)] $(i,j)\not\in V_{C_{\lambda}}$: One can verify that $(\tilde{C}_{\lambda}+N)_{ij}=N_{ij}$.
Therefore, due to Condition (1-iii), we have 
\begin{equation}
\begin{aligned}|(\tilde{C}_{\lambda}+N)_{ij}| & \leq\beta\left(V_{C_{\lambda}},\|\tilde{C}_{\lambda}\|_{\max}\right)\\
 & \leq\min_{(k,l)\in V^{(c)}}\frac{\lambda_{kl}-|C_{kl}|}{\sqrt{C_{kk}\times C_{ll}}}\\
 & =\min_{(k,l)\in V^{(c)}}\tilde{\lambda}_{kl}-|\tilde{C}_{kl}|
\end{aligned}
\label{eq39}
\end{equation}
This leads to 
\begin{equation}
|(\tilde{C}_{\lambda}+N)_{ij}-\tilde{C}_{ij}|\leq|(\tilde{C}_{\lambda}+N)_{ij}|+|\tilde{C}_{ij}|\leq\min_{(k,l)\in V^{(c)}}(\tilde{\lambda}_{kl}-|\tilde{C}_{kl}|)+|\tilde{C}_{ij}|\leq\tilde{\lambda}_{ij}
\end{equation}
\end{itemize}
Therefore, it can be concluded that $(\tilde{C}_{\lambda}+N)^{-1}$
satisfies the KKT conditions for~\eqref{RGL2}. On the other hand,
note that $V_{(\tilde{C}_{\lambda}+N)^{-1}}=V_{\tilde{C}_{\lambda}}$
and the nonzero off-diagonal elements of $(\tilde{C}_{\lambda}+N)^{-1})$
and $\tilde{C}_{\lambda}$ have opposite signs. This concludes the
proof.~\hfill{}$\square$

\section{\label{sec:appenCond}Solving the Newton Subproblem in $O(1)$ CG
Iterations}

Let $\S^{n}$ be the set of $n\times n$ real symmetric matrices.
Given a sparsity pattern $V$, we define $\S_{V}^{n}\subseteq\S^{n}$
as the set of $n\times n$ real symmetric matrices with this sparsity
pattern. We consider the following minimization problem
\begin{equation}
\hat{y}\equiv\arg\min_{y\in\R^{m}}g(y)\equiv f_{*}(C-A(y)).\label{eq:orig_prob-1}
\end{equation}
Here, the problem data $A:\R^{m}\to\S_{F}^{n}$ is an orthogonal basis
for a sub-sparsity pattern $F\subset V$ that excludes the matrix
$C$. In other words, the operator $A$ satisfies
\begin{align*}
A(A^{T}(X)) & =P_{F}(X)\qquad\forall X\in S_{V}^{n}, & A^{T}(C) & =0.
\end{align*}
The penalty function $f_{*}$ is the convex conjugate of the ``log-det''
barrier function on $\S_{V}^{n}$:

\[
f_{*}(S)=-\min_{X\in\S_{V}^{n}}\{S\bullet X-\log\det X\}
\]
Assuming that $V$ is \emph{chordal}, the function $f_{*}(S)$, its
gradient $\nabla f_{*}(S)$, and its Hessian matrix-vector product
$\nabla^{2}f_{*}(S)[Y]$ can all be evaluated in closed-form~\cite{dahl2008covariance,Andersen2010,andersen2013logarithmic};
see also~\cite{vandenberghe2015chordal}. Furthermore, if the pattern
is \emph{sparse}, i.e. its number of elements in the pattern satisfy
$|V|=O(n)$, then all of these operations can be performed to arbitrary
accuracy in $O(n)$ time and memory.

It is standard to solve (\ref{eq:orig_prob-1}) using Newton's method.
Starting from some initial point $y_{0}\in\dom g$
\[
y_{k+1}=y_{k}+\alpha_{k}\Delta y_{k}\qquad\Delta y_{k}\equiv-\nabla^{2}g(y_{k})^{-1}\nabla g(y_{k}),
\]
in which the step-size $\alpha_{k}$ is determined by backtracking
line search, selecting the first instance of the sequence $\{1,\rho,\rho^{2},\rho^{3},\dots\}$
that satisfies the Armijo\textendash Goldstein condition
\[
g(y+\alpha\Delta y)\le g(y)+\gamma\alpha\Delta y^{T}\nabla g(y),
\]
in which $\gamma\in(0,0.5)$ and $\rho\in(0,1)$. The function $f_{*}$
is strongly self-concordant, and $g$ inherits this property from
$f_{*}$. Accordingly, classical analysis shows that we require at
most 
\[
\left\lceil \frac{g(y_{0})-g(\hat{y})}{0.05\gamma\rho}+\log_{2}\log_{2}(1/\epsilon)\right\rceil \approx O(1)\text{ Newton steps}
\]
to an $\epsilon$-optimal point satisfying $g(y_{k})-g(\hat{y})\le\epsilon$.
The bound is very pessimistic, and in practice, no more than 20-30
Newton steps are ever needed for convergence.

The most computationally expensive of Newton's method is the solution
of the Newton direction $\Delta y$ via the $m\times m$ system of
equations
\begin{equation}
\nabla^{2}g(y_{k})\Delta y=-\nabla g(y_{k}).\label{eq:newton_sys-1}
\end{equation}
The main result in this section is a proof that the condition number
of $\nabla^{2}g(y)$ is independent of the problem dimension $n$.

\condbound*

As a consequence, each Newton direction can be computed in $O(\log\epsilon^{-1})$
iterations using conjugate gradients, over $O(\log\log\epsilon^{-1})$
total Newton steps. The overall minimization problem is solved to
$\epsilon$-accuracy in
\[
O(\log\epsilon^{-1}\log\log\epsilon^{-1})\approx O(1)\text{ CG iterations.}
\]
The leading constant here is dependent polynomially on the problem
data and the quality of the initial point, but \emph{independent of
the problem dimensions}.

\subsection{Preliminaries}

We endow $\S_{V}^{n}$ with the usual matrix inner product $X\bullet Y=\tr XY$
and associated Euclidean norm $\|X\|_{F}^{2}=X\bullet X$. The projection
\[
P_{V}:\S^{n}\to\S_{V}^{n}\qquad P_{V}(X)=\arg\min_{Y\in\S_{V}^{n}}\|X-Y\|_{F}^{2}
\]
is the associated projection from $\S^{n}$ onto $\S_{V}^{n}$. 

Let $\S_{+}^{n}\subset\S^{n}$ be the set of positive semidefinite
matrices. (We will frequently write $X\succeq0$ to mean $X\in\S_{+}^{n}$.)
We define the cone of \emph{sparse positive semidefinite matrices},
\[
\K=\S_{+}^{n}\cap\S_{V}^{n}
\]
and its dual cone, the cone of \emph{sparse matrices with positive
semidefinite completions,}
\[
\K_{*}=\{S\bullet X\ge0:S\in\S_{V}\}=P_{V}(\S_{+}^{n}).
\]
Being dual cones, $\K$ and $\K_{*}$ satisfy Farkas' lemma.
\begin{lem}[Farkas' lemma]
Given an arbitrary $Y\in\S_{V}^{n}$
\begin{enumerate}
\item Either $Y\in\K$, or there exists a separating hyperplane $S\in\K^{*}$
such that $S\bullet Y<0$.
\item Either $Y\in\K_{*}$, or there exists a separating hyperplane $X\in\K$
such that $Y\bullet X<0$.
\end{enumerate}
\end{lem}
The following barrier functions on $\K$ and $\K_{*}$ were introduced
by Dahl, Andersen and Vandenberghe~\cite{dahl2008covariance,andersen2013logarithmic,Andersen2010}:
\[
f(X)=-\log\det X,\qquad f_{*}(S)=-\min_{X\in\S_{V}^{n}}\{S\bullet X-\log\det X\}.
\]
The gradients of $f$ are simply the projections of their usual values
onto $\S_{V}^{n}$, as in
\begin{gather*}
\nabla f(X)=-P_{V}(X^{-1}),\qquad\nabla^{2}f(X)[Y]=P_{V}(X^{-1}YX^{-1}),
\end{gather*}
Then we also have for any $S\in\K_{*}$
\[
f_{*}(S)=n+\log\det X,\qquad\nabla f(S)=-X,\qquad\nabla^{2}f(S)[Y]=\nabla^{2}f(X)^{-1}[Y],
\]
where $X\in\K$ is the unique matrix satisfying $P_{V}(X^{-1})=S$.
Dahl, Andersen and Vandenberghe~\cite{dahl2008covariance,andersen2013logarithmic,Andersen2010}
showed that all six operations defined above can be efficiently evaluated
in $O(n)$ time on a \emph{sparse }and \emph{chordal} sparsity pattern
$V$.

From the above, we see that the Hessian matrix $\nabla^{2}g(y)$ can
be written it terms of the Hessian $\nabla^{2}f(X)$ and the unique
$X\in\K$ satisfying $P_{V}(X^{-1})=C-A(y)$, as in
\begin{align*}
\nabla^{2}g(y) & =A^{T}(\nabla^{2}f(X)^{-1}[A(y)])=\A^{T}\nabla^{2}f(X)^{-1}\A,
\end{align*}
in which $\A=[\vector A_{1},\ldots,\vector A_{m}]$. Moreover, using
the theory of Kronecker products, we can write
\[
\vector\nabla^{2}f(X)[Y]=Q^{T}(X^{-1}\otimes X^{-1})Q\vector Y
\]
in which the $\frac{1}{2}n(n+1)\times|V|$ matrix $Q$ is the orthogonal
basis matrix of $\S_{V}^{n}$ in $\S^{n}$. Because of this, we see
that the eigenvalues of the Hessian $\nabla^{2}g(y)$ are bound
\[
\lambda_{\min}(\A^{T}\A)\lambda_{\min}^{2}(X^{-1})\le\lambda_{i}(\nabla^{2}g(y))\le\lambda_{\max}(\A^{T}\A)\lambda_{\max}^{2}(X^{-1}),
\]
and therefore its condition number is bound by the eigenvalues of
$X$
\begin{equation}
\cond(\nabla^{2}g(y))\le\cond(\A^{T}\A)\left(\frac{\lambda_{\max}(X)}{\lambda_{\min}(X)}\right)^{2}=\left(\frac{\lambda_{\max}(X)}{\lambda_{\min}(X)}\right)^{2}.\label{eq:bnd_g_X}
\end{equation}
The second equality arises here because $A$ is orthogonal by construction;
given that it is defined to satisfy $A(A^{T}(X))=P_{F}(X)$ for some
sparsity pattern $F\subset V$, we must have $A^{T}(A(y))=y$ for
all $y\in\R^{m}$. Consequently most of our effort will be expended
in bounding the eigenvalues of $X$.

\subsection{Proof of Theorem~\ref{thm:cond_bound}}

To simplify notation, we will write $y_{0}$, $\hat{y}$, and $y$
as the initial point, the solution, and any $k$-th iterate. From
this, we define $S_{0},$ $\hat{S},$ $S,$ with each satisfying $S=C-A(y)$,
and $X_{0},$ $\hat{X},$ and $X$, the points in $\K$ each satisfying
$P_{V}(X^{-1})=S$. 

Our goal is to bound the extremal eigenvalues of $X$. To do this,
we first recall that the sequence is monotonously decreasing by hypothesis,
as in
\[
g(\hat{y})\le g(y)\le g(y_{0}).
\]
Evaluating each $f_{*}(S)$ as $n+\log\det X$, yields
\begin{equation}
\log\det\hat{X}\le\log\det X\le\log\det X_{0}.\label{eq:subopt_bound}
\end{equation}

Next, we introduce the following function, which often appears in
the study of interior-point methods
\[
\phi(M)=\tr M-\log\det M-n\ge0,
\]
and is well-known to provide a control on the arthmetic and geometric
means of the eigenvalues of $M$. Indeed, the function attains its
unique minimum at $\phi(I)=0$, and it is nonnegative precisely because
of the arithmetic-geometric inequality. Let us show that it can also
bound the arithmetic-geometric means of the extremal eigenvalues of
$M$. 
\begin{lem}
Denote the $n$ eigevalues of $M$ as $\lambda_{1}\ge\cdots\ge\lambda_{n}$.
Then
\[
\phi(M)\ge\lambda_{1}+\lambda_{n}-2\sqrt{\lambda_{1}\lambda_{n}}=(\sqrt{\lambda_{1}}-\sqrt{\lambda_{n}})^{2}.
\]
\end{lem}
\begin{proof}
Noting that $x-\log x-1\ge0$ for all $x\ge0$, we have
\begin{align*}
\phi(M) & =\sum_{i=1}^{n}(\lambda_{i}-\log\lambda_{i}-1)\\
 & \ge(\lambda_{1}-\log\lambda_{1}-1)+(\lambda_{n}-\log\lambda_{n}-1)\\
 & =\lambda_{1}+\lambda_{n}-2\log\sqrt{\lambda_{1}\lambda_{n}}-2\\
 & =\lambda_{1}+\lambda_{n}-2\sqrt{\lambda_{1}\lambda_{n}}+2(\sqrt{\lambda_{1}\lambda_{n}}-\log\sqrt{\lambda_{1}\lambda_{n}}-1)\\
 & \ge\lambda_{1}+\lambda_{n}-2\sqrt{\lambda_{1}\lambda_{n}}.
\end{align*}
Completing the square yields $\phi(M)\ge(\sqrt{\lambda_{1}}-\sqrt{\lambda_{n}})^{2}$.
\end{proof}
The following upper-bounds are the specific to our problem, and are
the key to our intended final claim.
\begin{lem}
Define the initial suboptimality $\phi_{\max}=\log\det X_{0}-\log\det\hat{X}$.
Let $\nabla g(y)^{T}(y-y_{0})\le\phi_{\max}$. Then we have
\[
\phi(\hat{X}X^{-1})\le\phi_{\max},\qquad\phi(XX_{0}^{-1})\le2\phi_{\max}.
\]
\end{lem}
\begin{proof}
To prove the first inequality, we take first-order optimality at the
optimal point $\hat{y}$
\[
\nabla g(\hat{y})=A^{T}(\hat{X})=0.
\]
Noting that $\hat{X}\in\S_{V}^{n}$, we further have
\begin{align*}
X^{-1}\bullet\hat{X}=P_{V}(X^{-1})\bullet\hat{X} & =[C-A(y)]\bullet\hat{X}=C\bullet\hat{X}-y^{T}\underbrace{A^{T}(\hat{X})}_{=0}\\
 & =[P_{V}(\hat{X}^{-1})+A(\hat{y})]\bullet\hat{X}\\
 & =P_{V}(\hat{X}^{-1})\bullet\hat{X}=n
\end{align*}
and hence $\phi(X^{-1}\hat{X})$ has the value of the suboptimality
at $X$, which is bound by the initial suboptimality in (\ref{eq:subopt_bound}):
\begin{align*}
\phi(X^{-1}\hat{X}) & =\underbrace{X^{-1}\bullet\hat{X}}_{n}-\log\det X^{-1}\hat{X}-n\\
 & =\log\det X-\log\det\hat{X}\\
 & \le\log\det X_{0}-\log\det\hat{X}=\phi_{\max}.
\end{align*}
We begin with the same steps to prove the second inequality:
\begin{align*}
X_{0}^{-1}\bullet X=P_{V}(X_{0}^{-1})\bullet X & =[C-A(y_{0})]\bullet X\\
 & =[P_{V}(X^{-1})+A(\hat{y})]\bullet X-A(y_{0})\bullet X\\
 & =n+A(y-y_{0})\bullet X.
\end{align*}
Note that $\nabla g(y)^{T}(y-y_{0})=(y-y_{0})^{T}A^{T}(X)=A(y-y_{0})\bullet X\le\phi_{\max}.$
Substituting and applying (\ref{eq:subopt_bound}) yields
\begin{align*}
\phi(X_{0}^{-1}X) & =X_{0}^{-1}\bullet X-\log\det X_{0}^{-1}X-n\\
 & =(n+A(y-y_{0})\bullet X-\log\det X_{0}^{-1}X)-n\\
 & =\underbrace{\log\det X_{0}X^{-1}}_{\le\phi_{\max}}+\underbrace{A(y-y_{0})\bullet X}_{\le\phi_{\max}}.
\end{align*}
\end{proof}
Using the two upper-bounds to bound the eigenvalues of their arguments
is enough to derive a condition number bound on $X$, which immediately
translates into a condition number bound on $\nabla^{2}g(y)$. 
\begin{proof}[Proof of Theorem~\ref{thm:cond_bound}]
To prove the first bound (\ref{eq:thm_bnd1}), we will instead prove
\begin{equation}
\frac{\lambda_{\max}(X)}{\lambda_{\min}(X)}\le2+\frac{2\phi_{\max}^{2}\lambda_{\max}(X_{0})}{\lambda_{\min}(\hat{X})},\label{eq:thm_bnd3}
\end{equation}
which yields the desired condition number bound on $\nabla^{2}g(y)$
by substituting into (\ref{eq:bnd_g_X}). Writing $\lambda_{1}=\lambda_{\max}(X)$
and $\lambda_{n}=\lambda_{\min}(X)$, we have from the two lemmas
above:
\begin{align*}
\phi_{\max} & \ge\lambda_{\min}(\hat{X})(\sqrt{\lambda_{n}^{-1}}-\sqrt{\lambda_{1}^{-1}})^{2}>0,\\
2\phi_{\max} & \ge\lambda_{\min}(X_{0}^{-1})(\sqrt{\lambda_{1}}-\sqrt{\lambda_{n}})^{2}>0.
\end{align*}
Multiplying the two upper-bounds and substituing $\lambda_{\min}(X_{0}^{-1})=1/\lambda_{\max}(X_{0})$
yields
\[
\frac{2\phi_{\max}^{2}\lambda_{\max}(X_{0})}{\lambda_{\min}(\hat{X})}\ge\left(\sqrt{\frac{\lambda_{1}}{\lambda_{n}}}-\sqrt{\frac{\lambda_{n}}{\lambda_{1}}}\right)^{2}=\frac{\lambda_{1}}{\lambda_{n}}+\frac{\lambda_{n}}{\lambda_{1}}-2.
\]
Finally, bounding $\lambda_{n}/\lambda_{1}\ge0$ yields (\ref{eq:thm_bnd3}). 
\end{proof}

\end{document}